\def\year{2021}\relax
\documentclass[letterpaper]{article} 
\usepackage{aaai21}  
\usepackage{times}  
\usepackage{helvet} 
\usepackage{courier}  
\usepackage[hyphens]{url}  
\usepackage{graphicx} 
\usepackage{natbib}  
\usepackage{caption} 
\usepackage{siunitx}
\newcommand{\mytitle}{Planning with Learned Object Importance in Large Problem Instances using Graph Neural Networks} 
\newcommand{\mytitlewithbreaks}{Planning with Learned Object Importance in Large Problem Instances\\using Graph Neural Networks} 

\title{\mytitlewithbreaks}
\author{Tom Silver$^*$, Rohan Chitnis$^*$, Aidan Curtis,\\ Joshua Tenenbaum, Tom\'{a}s Lozano-P\'{e}rez, Leslie Pack Kaelbling\\\textnormal{MIT Computer Science and Artificial Intelligence Laboratory}\\\texttt{\textnormal{\{tslvr, ronuchit, curtisa, jbt, tlp, lpk\}@mit.edu}}}

\usepackage{multicol}
\usepackage{color}
\usepackage{amssymb}
\usepackage{amsmath}
\usepackage{amsthm}
\usepackage{tikz}
\usepackage{chngcntr}
\usepackage[makeroom]{cancel}
\usetikzlibrary{arrows,decorations.markings}

\usepackage{booktabs}
\usepackage{graphicx}
\usepackage{mathtools}
\usepackage{bbm}
\usepackage[switch]{lineno}

\usepackage{caption}
\usepackage{subcaption}

\usepackage{algorithm2e}
\let\oldnl\nl
\newcommand{\nonl}{\renewcommand{\nl}{\let\nl\oldnl}}

\usepackage{indentfirst}

\newenvironment{tightlist}%
{\begin{list}{$\bullet$}{%
    \setlength{\topsep}{0in}
    \setlength{\partopsep}{0in}
    \setlength{\itemsep}{0in}
    \setlength{\parsep}{0in}
    \setlength{\leftmargin}{1.5em}
    \setlength{\rightmargin}{0in}
}
}%
{\end{list}
}

\newcommand{\secref}[1]{Section \ref{#1}}

\newcommand{\figref}[1]{Figure~\ref{#1}}
\newcommand{\algref}[1]{Algorithm~\ref{#1}}

\newcommand{\tabref}[1]{Table~\ref{#1}}

\newcommand{\defref}[1]{Definition~\ref{#1}}
\newcommand{\appref}[1]{Appendix~\ref{#1}}

\newtheorem{lem}{Lemma}
\newtheorem{defn}{Definition}

\counterwithin*{thm}{section}

\def\thickhline{%
  \noalign{\ifnum0=`}\fi\hrule \@height \thickarrayrulewidth \futurelet
   \reserved@a\@xthickhline}
\def\@xthickhline{\ifx\reserved@a\thickhline
               \vskip\doublerulesep
               \vskip-\thickarrayrulewidth
             \fi
      \ifnum0=`{\fi}}

\newlength{\thickarrayrulewidth}
\newcommand{\G}{\mathcal{G}}
\renewcommand{\P}{\mathcal{P}}
\renewcommand{\O}{\mathcal{O}}

\renewcommand{\S}{\mathcal{S}}
\newcommand{\A}{\mathcal{A}}

\newcommand{\V}{\mathcal{V}}
\newcommand{\E}{\mathcal{E}}

\newcommand{\plan}{{\sc plan}}
\newcommand{\ploi}{{\sc ploi}}

\DeclarePairedDelimiterX{\infdivx}[2]{(}{)}{%
  #1\;\delimsize\|\;#2%
}

\begin{document}
\maketitle

\begin{abstract}
Real-world planning problems often involve hundreds or even thousands of objects, straining the limits of modern planners.
In this work, we address this challenge by learning to predict a small set of objects that, taken together, would be sufficient for finding a plan.
We propose a graph neural network architecture for predicting object importance in a single inference pass, thus incurring little overhead while greatly reducing the number of objects that must be considered by the planner. Our approach treats the planner and transition model as black boxes, and can be used with any off-the-shelf planner. Empirically, across classical planning, probabilistic planning, and robotic task and motion planning, we find that our method results in planning that is significantly faster than several baselines, including other partial grounding strategies and lifted planners. We conclude that learning to predict a sufficient set of objects for a planning problem is a simple, powerful, and general mechanism for planning in large instances. Video: \url{https://youtu.be/FWsVJc2fvCE} Code: \url{https://git.io/JIsqX}
\end{abstract}

\section{Introduction}
\label{sec:intro}

A key research agenda in classical planning is to extend the core framework to large-scale real-world applications. 
Such applications will often involve many objects, only some of which are important for any particular goal. 
For example, a household robot's internal state must include all objects relevant to \emph{any} of its functions, but once it receives a \emph{specific} goal, such as boiling potatoes, it should restrict its attention to only a small object set, such as the potatoes, pots, and forks, ignoring the hundreds or even thousands of other objects.
If its goal were instead to clean the sink, the set of objects to consider would vary drastically.

\begin{figure}[t]
  \centering
    \noindent
    \includegraphics[width=\columnwidth]{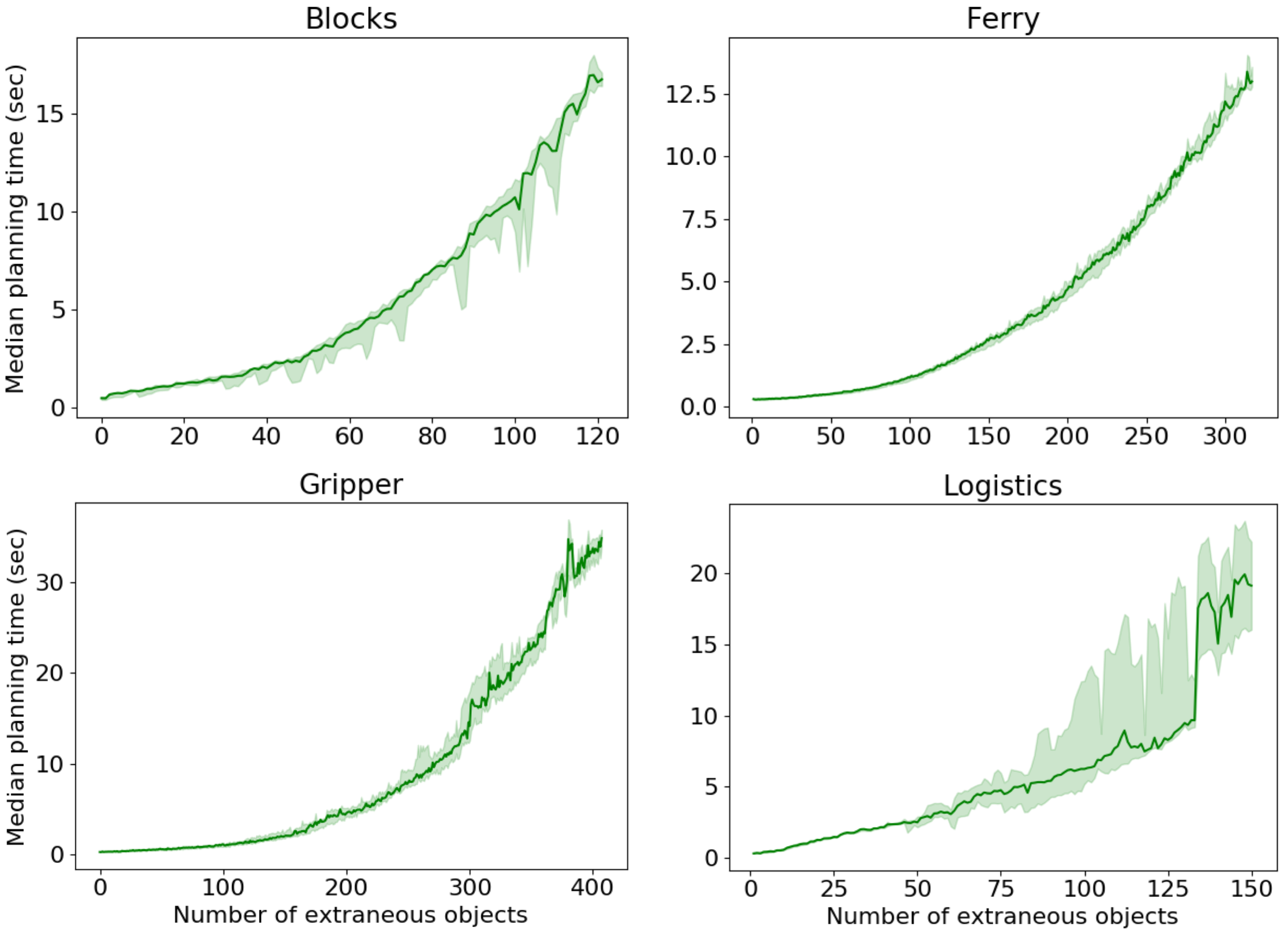}
    \caption{Time taken by Fast Downward~\cite{fd} in the \texttt{LAMA-first} mode on various IPC domains, as a function of the number of extraneous objects in the problem. The $x$-axis is the number of objects added to a small sufficient set (\defref{def:sufficiency}). Curves show a median across 10 problems; shaded regions show the 25th to 75th percentiles. We can see that planning time gets substantially worse as the number of extraneous objects increases; real-world applications of planning will often contain large numbers of such objects for a particular goal. In this work, we learn to predict a small subset of objects that is sufficient for planning, leading to significantly faster planning than both Fast Downward on its own and other learning-based grounding methods.}
  \label{fig:teaser}
\end{figure}

More generally, we consider planning problems with large (but finite) universes of objects, where only a small subset need to be considered for any particular goal.
Popular heuristic search planners~\cite{ff,fd,iw} scale poorly in this regime (\figref{fig:teaser}), as they ground actions over the objects during preprocessing.
Lifted planners~\cite{ridder2014lifted,correa2020lifted} avoid explicit grounding, but struggle during search; we find that a state-of-the-art lifted planner~\cite{correa2020lifted} fails to solve any test problem in our experiments within the timeout (but it can usually solve the much smaller training problems).
In this many-object setting, one would instead like to identify a small sufficient set of objects \emph{before} planning, but finding this set is nontrivial and highly problem-dependent. 

In this work, we propose to \emph{learn to predict} subsets of objects that are sufficient for solving planning problems.
This requires reasoning about discrete and continuous properties of the objects and their relations. Generalizing to problems with more objects requires learning lifted models that are agnostic to object identity and count.
We therefore propose a convolutional graph neural network architecture \cite{scarselli2008graph,kipf2016semi,battaglia2018relational} that learns from a modest number $(< 50)$ of small training problems and generalizes to hard test problems with many more objects.
On the test problems, we use the network to predict a sufficient object set, remove all facts from the initial state and goal referencing excluded objects, and call an off-the-shelf planner on this reduced planning problem.
For completeness, we wrap this procedure in an incremental loop that considers more objects until a solution is found.

Object importance prediction offers several advantages over alternative learning-based approaches: 
(1) it can treat the planner and transition model as black boxes; 
(2) its runtime does not depend on the number of ground actions (for a constant number of objects);  
(3) it permits efficient inference, therefore contributing negligibly to the overall planning time; and 
(4) it allows for a large margin of error in one direction, since the planning time can improve substantially even if only some irrelevant objects are excluded (\figref{fig:teaser}).

Gathering training data can be challenging in this setting because it requires labels of which objects are relevant; it would be impractical to assume that such labels are given. Instead, we propose a greedy approximate procedure for generating these labels automatically, which is only conducted in the relatively small training problems.

In experiments, we consider classical planning, probabilistic planning, and robotic task and motion planning, with test problems containing hundreds or thousands of objects.
Our method, \textbf{P}lanning with \textbf{L}earned \textbf{O}bject \textbf{I}mportance (\ploi{}), results in planning that is much more efficient than several baselines, including policy learning~\cite{groshev2017learning,rivlin2020generalized} and partial action grounding~\cite{gnad2019learning}.
We conclude that object importance prediction is a simple, powerful, and general mechanism for planning in large instances with many objects.
\section{Related Work}
\label{sec:related}

\subsubsection{Planning with Many Objects.}

Planning for problem instances that contain many objects is one of the main motivations for ongoing research in \emph{lifted planning} \cite{ridder2014lifted,correa2020lifted}.
In STRIPS-like domains, lifted planners avoid the expensive preprocessing step of grounding the actions over all objects.
Another way to alleviate the burden of grounding is to simplify the planning problem by creating \emph{abstractions} \cite{dearden1997abstraction,dietterich2000state,gardiolthesis,abel2017near}.
Our object importance predictor can also be viewed as a type of learned abstraction selection \cite{konidarisabstractions,riddle2016improving,haslum2007reducing}.
    
\subsubsection{Relational Representations for Learning to Plan.} 

Our work uses graph neural networks (GNNs)~\cite{scarselli2008graph,kipf2016semi,battaglia2018relational}, an increasingly popular choice for relational machine learning with applications to planning \cite{wu2020comprehensive,ma2020online,shen2020learning,rivlin2020generalized}.
One advantage of GNNs over logical representations \cite{ilp1,ilp2,dvzeroski2001relational} is that GNNs natively support continuous object-level and relational properties.
We make use of this flexibility in our experiments, showing results in a simulated robotic environment.

\subsubsection{Generalized Planning.}
Our work may be seen as an instance of \emph{generalized planning}, which broadly encompasses methods for collectively solving a set of planning problems, rather than a single problem in isolation \cite{jimenez2019review}.
Other approaches to generalized planning include generalized policy learning \cite{triangletables,groshev2017learning,gomoluch2019learning}, 
incremental search \cite{koenig2004incremental,pommerening2013incremental}, and
heuristic or value function learning \cite{yoon2008learning,arfaee2011learning,silver2016mastering,shen2020learning}.
Incremental search and heuristic learning are complementary to our work and could be easily combined; generalized policy learning suggests a different mode of execution (executing the policy without planning) and we therefore include it as a baseline in our experiments.

The work perhaps most similar to ours is that of \citet{gnad2019learning}, who propose partial action grounding as another approach to generalized planning in large problems.
Rather than predicting the probability that \emph{objects} will be included in a plan (as we do), their approach predicts the probability that \emph{ground actions} will be included.
We include two versions of this approach as baselines in our experiments, including the implementation provided by the authors.
\section{Problem Setup}
\label{sec:setup}

We now give background and describe our problem setup.

A \emph{property} is a real-valued function on a tuple of \emph{objects}. For example, in the expression \texttt{pose(cup3) = 5.7}, the property is \texttt{pose} and the tuple of objects is $\langle \texttt{cup3} \rangle$. Predicates, e.g., \texttt{on} in the expression \texttt{on(cup3, table) = True}, are a special case of properties where the output is binary. For simplicity, we assume properties have arity at most 2; higher-order ones can often be converted to an equivalent set of binary (arity 2) properties~\cite{rivlin2020generalized}. We treat object types as unary (arity 1) properties.

A \emph{planning problem} is a tuple $\Pi = \langle \P, \A, T, \O, I, G \rangle$, where $\P$ is a finite set of properties, $\A$ is a finite set of object-parameterized actions, $T$ is a (possibly stochastic) transition model, $\O$ is a finite set of objects, $I$ is the initial state, and $G$ is the goal. A state is an assignment of values to all possible applications of properties in $\P$ with objects in $\O$.
A goal is an assignment of values to any subset of the ground properties, which implicitly represents a set of states.
We use $\S$ to denote the set of possible states and $\G$ to denote the set of possible goals over $\P$.
A ground action results from applying an object-parameterized action in $\A$ to a tuple of objects in $\O$; for example, \texttt{pick(?x)} is an object-parameterized action and \texttt{pick(cup3)} is a ground action. The transition model $T$ defines the dynamics of the environment; it maps a state, ground action, and next state to a probability.

We focus on planning problems with extraneous objects: ones that, if ignored, would make planning easier. The methods we propose are biased toward this subclass of planning problems and would not offer benefits in problems for which planning is easier, or only feasible, with all objects.

We consider the usual learning setting where we are first given a set of \emph{training problems}, and then a separate set of \emph{test problems}. All problems share $\P$, $\A$, and $T$, but may have different $\O$, $I$, and $G$. In general, the test problems will have a much larger set of objects $\O$ than the training problems.

We are also given a black-box \emph{planner}, denoted \plan{}, which given a planning problem $\Pi$ as described above, produces either (1) a plan (a sequence of ground actions) if $T$ is deterministic; or (2) a policy (a mapping from states to ground actions) if $T$ is stochastic. 
A plan is a \emph{solution} to $\Pi$ if following the actions from the initial state reaches a goal state.
A policy is a solution to $\Pi$ if executing the policy from the initial state reaches a goal state within some time horizon, with probability above some threshold; in practice, this can be approximated by sampling trajectories. 
Going forward, we will not continue to make this distinction between plans and policies; in either case, at an intuitive level, \plan{} produces ground actions that drive the agent toward its goal.

Our objective in this work is to maximize the number of test problems solved within some time budget.
Because the test problems contain many objects, and planners are often highly sensitive to this number, we will follow the broad approach of learning a model (on the training problems) that speeds up planning (on the test problems).

\section{Planning with Object Importance}
\label{sec:approach}

\begin{algorithm}[t]
  \SetAlgoNoEnd
  \DontPrintSemicolon
  \SetKwFunction{algo}{algo}\SetKwFunction{proc}{proc}
  \SetKwProg{myalg}{}{}{}
  \SetKwProg{myproc}{Subroutine}{}{}
  \SetKw{Continue}{continue}
  \SetKw{Break}{break}
  \SetKw{Return}{return}
  \myalg{\textsc{Planning with Learned Object Importance}}{
    \nonl \textbf{Input:} Planning problem $\Pi = \langle \P, \A, T, \O, I, G \rangle$. 
    \tcp*{\footnotesize See \secref{sec:setup}}
    \nonl \textbf{Input:} Object scorer $f$.
    \tcp*{\footnotesize See \secref{sec:approach}}
    \nonl \textbf{Hyperparameter:} Geometric threshold $\gamma$.\;
    \tcp{\footnotesize Step 1: compute importance scores}
    \nonl Compute $\mathrm{score}(o)$ = $f(o, I, G)\ \ \forall o \in \O$\;
    \tcp{\footnotesize Step 2: incremental planning}
    \nonl \For{$N = 1, 2, 3, ...$}
    {
    \tcp{\footnotesize Select objects above threshold}
    \nonl $\hat{\O} \gets \{ o : o \in \O, \mathrm{score}(o) \geq \gamma^N \}$\;
    \tcp{\footnotesize Create reduced problem \& plan}
    \nonl $\hat{\Pi} \gets \textsc{ReduceProblem}(\Pi, \hat{\O})$\;
    \nonl $\pi \gets$ \textsc{Plan}($\hat{\Pi}$)\;
    \tcp{\footnotesize Validate on original problem}
    \nonl \If{\textsc{IsSolution}($\pi, \Pi$) or $\hat{\O}=\O$}
    {
    \Return $\pi$
    }
    }
    }\;
\caption{\small{Pseudocode for \ploi{}. In practice, we perform two optimizations: (1) plan only when the object set $\hat{\O}$ changes, so that \plan{} is called at most $|\O|$ times; and (2) assign a score of 1 to all objects named in the goal. See \secref{sec:approach} for details and \figref{fig:ploi} for an example.}
}
\label{alg:ploi}
\end{algorithm}

In this section, we describe our approach for learning to plan efficiently in large problems.
Our main idea is to learn a model that predicts a sufficient subset of the full object set.
At test time, we use the learned model to construct a \emph{reduction} of the planning problem, plan in the reduction, and validate the resulting plan in the original problem.
To guarantee completeness, we repeat this procedure, incrementally growing the subset until a solution is found.
This overall method --- {\bf P}lanning with {\bf L}earned {\bf O}bject {\bf I}mportance (\ploi{}) --- is summarized in \algref{alg:ploi} and \figref{fig:ploi}.

\begin{figure}[t]
  \centering
    \noindent
    \includegraphics[width=\columnwidth]{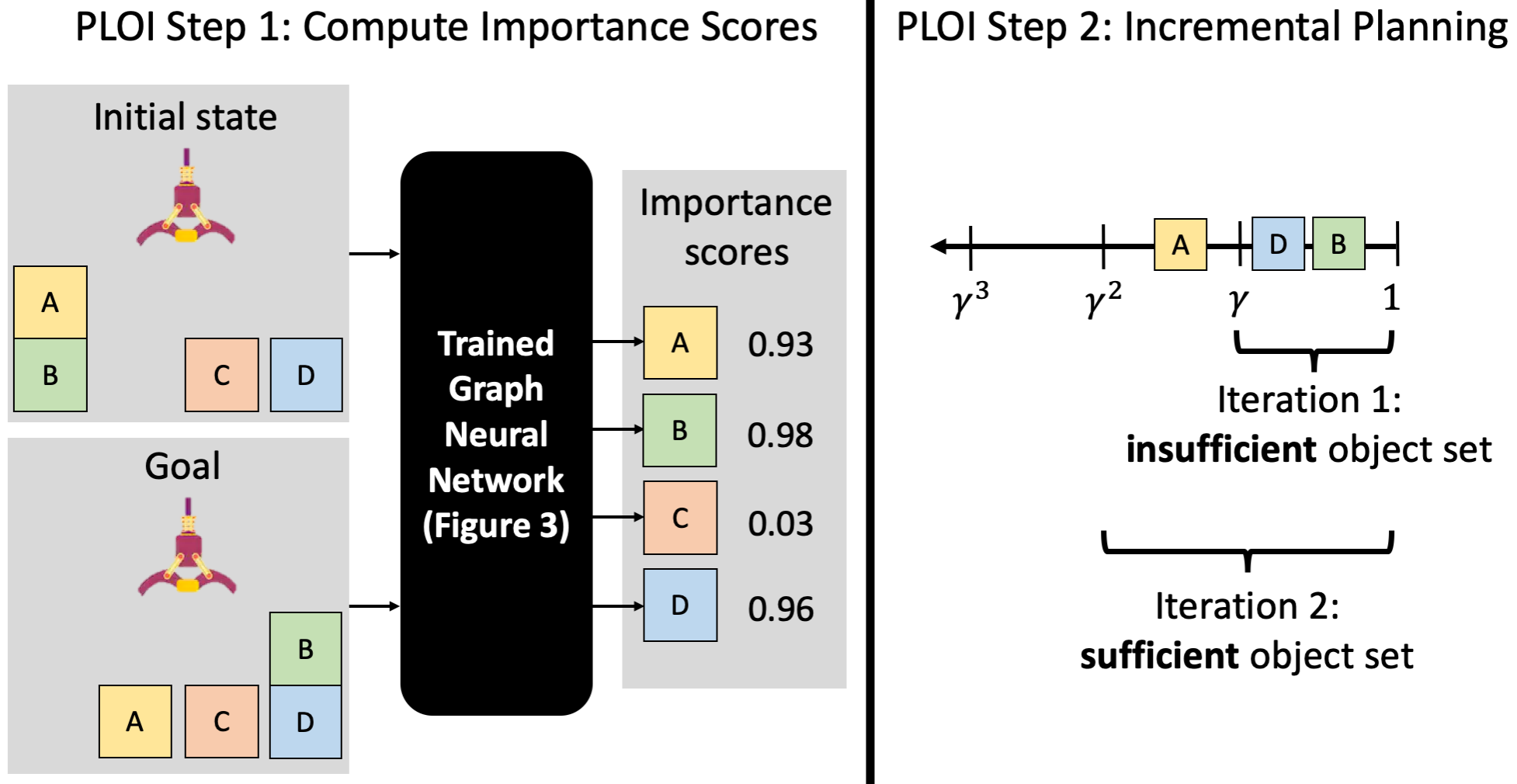}
    \caption{Overview of our method, \ploi{}, with an example. \emph{Left:} To solve this problem, the robot must move block A to the free space, then stack B onto D. The GNN computes the per-object importance score. Block C is irrelevant, and therefore it receives a low score of 0.03. \emph{Right:} We perform incremental planning. In this example, $\gamma=0.95$, so that $\gamma^2 \approx 0.9$. The first iteration tries planning with the object set \{B, D\}, which fails because it does not consider the obstructing A on top of B. The second iteration succeeds, because the object set \{A, B, D\} is sufficient for this problem.}
  \label{fig:ploi}
\end{figure}

We now describe \ploi{} in more detail, beginning with a more formal description of the reduced planning problem.
\begin{defn}[Object set reduction]
\label{def:reduction}
Given a planning problem $\Pi = \langle \P, \A, T, \O, I, G \rangle$ and a subset of objects $\hat{\O} \subseteq \O$, the \emph{problem reduction} $\hat{\Pi} = \textsc{ReduceProblem}(\Pi, \hat{\O})$ is given by $\hat{\Pi} = \langle \P, \A, T, \hat{\O}, \hat{I}, \hat{G} \rangle$, where $\hat{I}$ (resp. $\hat{G}$) is $I$ (resp. $G$) but with only properties over $\hat{\O}$.
\end{defn}
Intuitively, an object set reduction abstracts away all aspects of the initial state and goal pertaining to the excluded objects, and disallows any ground actions that involve these objects.
This can result in a dramatically simplified planning problem, but may also result in an oversimplification to the point where planning in the reduction results in an invalid solution, or no solution at all.
To distinguish such sets from the useful ones we seek, we use the following definition.
\begin{defn}[Sufficient object set]
\label{def:sufficiency}
Given a planning problem $\Pi = \langle \P, \A, T, \O, I, G \rangle$ and planner $\textsc{Plan}$, a subset of objects $\hat{\O} \subseteq \O$ is \emph{sufficient} if $\pi = \textsc{Plan}(\hat{\Pi})$ is a solution to $\Pi$, where $\hat{\Pi} = \textsc{ReduceProblem}(\Pi, \hat{\O})$.
\end{defn}
In words, an object set is \emph{sufficient} if planning in the corresponding reduction results in a valid solution for the original problem.
An object set that omits crucial objects, like a key needed to unlock a door or an obstacle that must be avoided, will not be sufficient: planning will fail without the key, and validation will fail without the obstacle.
Trivially, the complete set of objects $\O$ is always sufficient if the planning problem is satisfiable and the planner is complete.
However, we would like to identify a \emph{small} sufficient set that permits faster planning.
We therefore aim to learn a model that predicts such a set for a given initial state and goal.

\subsection{Scoring Object Importance Individually}
We wish to learn a model that allows us to identify a small sufficient subset of objects given an initial state, goal, and complete set of objects.
There are three basic requirements for such a model.
First, since our ultimate objective is to improve planning time, the model should be fast to query.
Second, since we want to optimize the model from a modest number of training problems, the model should permit data-efficient learning.
Finally, since we want to maintain completeness when the original planner is complete, the model should allow for some recourse when the first subset it predicts does not result in a valid solution.

These requirements preclude models that directly predict a subset of objects.
Such models offer no obvious recourse when the predicted subset turns out to be insufficient.
Moreover, models that reason about sets of objects are, in general, likely to require vast amounts of training data and may require exorbitant time during inference.

We instead choose to learn a model $f: \O \times \S \times \G \to (0, 1]$ that scores objects \emph{individually}. The output of the model $f(o, I, G)$ can be interpreted as the probability that the object $o$ will be included in a small sufficient set for the planning problem $\langle \P, \A, T, \O, I, G \rangle$.
We refer to this output score as the \emph{importance} of an object.
To get a candidate sufficient subset $\hat{\O}$ from such a model, we can simply take all objects with importance score above a threshold $0 < \gamma < 1$.

For the graph neural network architecture we will present in \secref{sec:gnn}, this inference is highly efficient, requiring only a single inference pass.
This parameterization also affords efficient learning, since as discussed at the end of this section, the loss function decomposes as a sum over objects. As an optimization, we always include in $\hat{\O}$ all objects named in the goal, since such objects must be in any sufficient set.

Another immediate advantage of predicting scores for objects individually is that there is natural recourse when the first candidate set $\hat{\O}$ does not succeed: simply lower the threshold $\gamma$ and retry.
In practice, we lower the threshold geometrically (see \algref{alg:ploi}), guaranteeing completeness.

\begin{lem}[\ploi{} is complete]
\label{def:completeness}
Given any object scorer $f : \O \times \S \times \G \to (0, 1]$, if the planner $\textsc{Plan}$ is complete, then \algref{alg:ploi} is complete.
\end{lem}
\begin{proof}
Since the codomain of $f$ excludes 0, there exists an $\epsilon > 0$ s.t. $\{ o : o \in \O, f(o, I, G) \ge \epsilon \} = \O$.
Furthermore, $0 < \gamma < 1$, so there exists an iteration $N$ s.t. $\gamma^N < \epsilon$.
Therefore, in the worst case, we will return $\textsc{Plan}$ on the original problem and return the result.
\end{proof}

In predicting scores for objects individually, we have made the set prediction problem tractable by restricting the hypothesis class, but it is important to note that this restriction makes it impossible to predict certain object subsets.
For example, in planning problems where a particular number of ``copies'' of the same object are required, e.g., three eggs in a recipe or five nails for assembly, individual object scoring can only predict the same score for all copies. 
In practice, we find that this limitation is sharply outweighed by the benefits of efficient learning and inference.

In \secref{sec:gnn}, we will present a graph neural network architecture for the object scorer $f$ that is well-suited for relational domains.
Before that, however, we describe a general methodology for learning $f$ on the set of training problems.

\subsection{Training with Supervised Learning}

We now describe a general method for learning an object scorer $f$ given a set of training problems ${\bf\Pi_{\text{train}}} = \{\Pi_1, \Pi_2, ..., \Pi_M\}$, where each $\Pi_i = \langle \P, \A, T, \O_i, I_i, G_i \rangle$.
The main idea is to cast the problem as supervised learning.
From each training problem $\Pi_i$, we want to extract input-output pairs $\{ ((o, I_i, G_i), y) \}$, where $o \in \O_i$ is each object from the full set for the problem, and $y \in \{0, 1\}$ is a binary label indicating whether $o$ should be predicted for inclusion in the small sufficient set.
The overall training dataset for supervised learning, then, will contain an input-output pair for every object, for each of the $M$ training problems.

The $y$ labels for the objects are \emph{not} given, and moreover, it can be challenging to exactly compute a minimal sufficient object set, even in small problem instances.
We propose a simple approximate method for automatically deriving the labels.
Given a training problem $\Pi_i$, we perform a greedy search over object sets: starting with the full object set $\O_i$, we iteratively remove an individual object from the set, accepting the new set if it is sufficient, until no more individual objects can be removed without violating sufficiency.
All objects in the final sufficient set are labeled with $y=1$, while the remaining objects are labeled with $y=0$.
This procedure, which requires planning several times per problem instance with full or near-full object sets to check sufficiency, takes advantage of the fact that the training problems are much smaller and easier than the test problems.

It should be noted that the aforementioned greedy procedure is an approximation, in the sense that there may be some smaller sufficient object set than the one returned.
To illustrate this point, consider a domain with a certain number of widgets where the only parameterized action is $\texttt{destroy(?widget)}$.
Suppose the goal is to be left with a number of widgets that is divisible by 10, and that the full object set itself has 10 widgets. The greedy procedure will terminate after the first iteration, since no object can be removed while maintaining sufficiency. However, the empty set is actually sufficient because it induces the empty plan, which trivially satisfies this goal.
Despite such possible cases, this greedy procedure for deriving the training data does well in practice to identify small sufficient object sets.

With a dataset for supervised learning in hand, we can proceed in the standard way by defining a loss function and optimizing model parameters.
To permit data-efficient learning, we use a loss function that decomposes over objects: $$\mathcal{L}({\bf\Pi_{\text{train}}}) = \sum_{i=1}^{M} \sum_{o_j \in \O_i} \mathcal{L}_{\text{obj}}(y_{ij}, f(o_j, I_i, G_i)),$$
where $y_{ij}$ is the binary label for the $j^{th}$ object in the $i^{th}$ training problem, and $f(o_j, I_i, G_i) \in (0, 1]$.
We use a weighted binary cross-entropy loss for $\mathcal{L}_{\text{obj}}$, where the weight (10 in experiments) gives higher penalty to false negatives than false positives, to account for class imbalance.

\section{Object Importance Scorers as GNNs}
\label{sec:gnn}

\begin{figure*}[t]
  \centering
    \noindent
    \includegraphics[width=\textwidth]{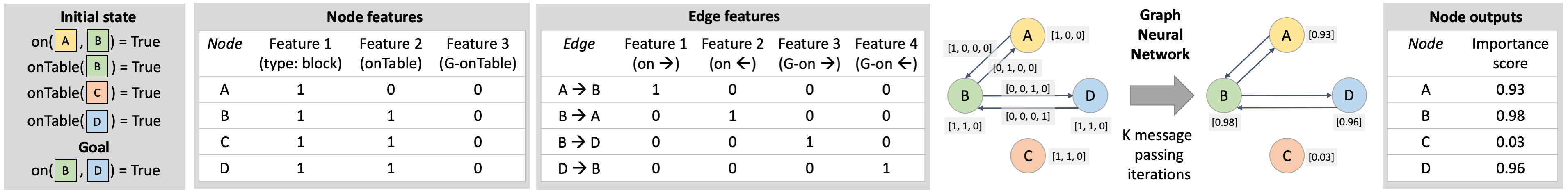}
    \caption{Illustration of object importance scoring with GNNs. 
    \emph{(Left to right)} We consider the same planning problem example as in \figref{fig:ploi}. A node is created for each of the four objects, with features determined by the unary properties in the initial state and goal.
    An edge is created for each ordered pair of objects, with features determined by the binary properties, and with trivial edges excluded.
    These nodes and edges constitute the input graph to a GNN, which performs $K=3$ message passing iterations before outputting another graph with the same topology. Each output node is associated with the object's importance score.
    }
  \label{fig:gnn}
\end{figure*}

We have established individual object importance scorers $f : \O \times \S \times G \to (0, 1]$ as the model that we wish to learn. We now turn to a specific model class that affords gradient-based optimization, data-efficient learning, and generalization to test problems with new and many more objects than were seen during training.
Graph neural networks (GNNs) offer a flexible and general framework for learning functions over graph-structured data \cite{kipf2016semi}.
GNNs employ a relational bias that is well-suited for our setting, where we want to make predictions based on the relations that objects are involved in, but we do not want to overfit to the particular identity or number of objects in the training problems \cite{battaglia2018relational}.
Such a relational bias is crucial for generalizing from training with few objects to testing with many. Furthermore, GNNs can be used in domains with continuous properties, unlike traditional inductive logic programming methods~\cite{ilp1,ilp2,ilp3}. We stress that other modeling choices are possible, such as statistical relational learning methods~\cite{koller2007introduction}, as long as they are lifted, relational, efficiently learnable, and able to handle continuous properties; we have chosen GNNs here because they are convenient and well-supported.

The input to a GNN is a directed graph with nodes $\V$ and edges $\E$.
Each node $v \in \V$ has a feature vector $\phi_{\text{node}}(v) \in \mathbb{R}^{D^{\text{in}}_{\text{node}}}$, where $D^{\text{in}}_{\text{node}}$ is the (common) dimensionality of these node feature vectors.
Each edge $(v_1, v_2) \in \E$ has a feature vector $\phi_{\text{edge}}(v_1, v_2) \in \mathbb{R}^{D^{\text{in}}_{\text{edge}}}$, where $D^{\text{in}}_{\text{edge}}$ is the (common) dimensionality of these edge feature vectors.
The output of a GNN is another graph with the same topology, but the node and edge features are of different dimensionalities: $D^{\text{out}}_{\text{node}}$ and $D^{\text{out}}_{\text{edge}}$ respectively.
Internally, the GNN passes messages for $K$ iterations from edges to sink nodes and from source nodes to edges, where the messages are determined by fully connected networks with weights shared across nodes and edges.
We use the standard Graph Network block~\cite{battaglia2018relational}, but other choices are possible.
Like other neural networks, GNNs can be trained with gradient descent.

We now describe how object importance scoring can be formulated as a GNN.
The high-level idea is to associate each object with a node, each unary property (including object types) with an input node feature, each binary property with an input edge feature, and each importance score with an output node feature.
See \figref{fig:gnn} for an example.

Given a planning problem with object set $\O$, we construct input and output graphs where each node $v \in \V$ corresponds to an object $o \in O$.
In the output graph, there is a single feature for each node; i.e., $D^{\text{out}}_{\text{node}} = 1$. This feature represents the importance score $f(o, I, G)$ of each object $o$.
The edges are ignored in the output graph.

The input graph is an encoding of the initial state $I$ and goal $G$.
Recall that the initial state $I$ is defined by an assignment of all ground properties $(\P$ over $\O)$ to values, and that all properties are unary (arity 1) or binary (arity 2).
Each unary property, which includes object types, corresponds to one dimension of the input node feature vector $\phi_{\text{node}}(o)$.
Each binary property corresponds to \emph{two} dimensions of the input edge feature vector $\phi_{\text{edge}}(o_1, o_2)$: one for each of the two orderings of the objects (see Figure \ref{fig:gnn}).

Recall that a goal $G$ is characterized by an assignment of some subset of ground properties to values.
Unlike the initial state, not all ground properties must appear in the goal; in practice, goals are typically very sparse relative to the state.
For each ground property, we must indicate whether it appears in the goal, and if so, with what assignment.
For each unary property, we add two dimensions to the input node feature vector $\phi_{\text{node}}(o)$: one indicating the presence (1) or absence (0) of the property, and the other indicating the value, with a default of 0 if the property is absent.
Similarly, for each binary property, we add four dimensions to the input edge feature vector $\phi_{\text{edge}}(o_1, o_2)$: two for the orderings multiplied by two for presence and assignment.

For STRIPS-like domains where properties are predicates, we make two small simplifications. 
First, to make the graph computations more efficient, we sparsify the edges by removing any edge whose features are all zeros.
Second, in the common case where goals do not involve negation, we note that the presence/absence dimension will be equivalent to the assignment dimension; we thus remove the redundant dimension.
\figref{fig:gnn} makes use of these simplifications.

Given a test problem and trained GNN, we construct an input graph, feed it to the GNN to get an output graph, and read off the predicted importance scores for all objects.
This entire procedure needs only one inference pass (with $K=3$ message passing iterations) to predict all object scores; it takes just a few milliseconds in our experiments.

\begin{table*}
	\centering
	\footnotesize
	\begin{tabular}{| l | p{0.7cm} | p{0.7cm} | p{0.7cm} | p{0.7cm} | p{0.7cm} | p{0.7cm} | p{0.7cm} | p{0.7cm} | p{0.7cm} | p{0.7cm} | p{0.7cm} | p{0.7cm} | p{0.7cm} | p{0.7cm} | }
	\hline
	\multicolumn{1}{|c|}{} &\multicolumn{2}{c|}{Pure Plan} &
	\multicolumn{2}{c|}{\ploi{} (Ours)} &
	\multicolumn{2}{c|}{Rand Score} &
	\multicolumn{2}{c|}{Neighbors} &
	\multicolumn{2}{c|}{Policy} &
	\multicolumn{2}{c|}{ILP AG} &
	\multicolumn{2}{c|}{GNN AG} \\
	\hline
	Domains &
	{\scriptsize Time} & {\scriptsize Fail} &
	{\scriptsize Time} & {\scriptsize Fail} &
	{\scriptsize Time} & {\scriptsize Fail} &
	{\scriptsize Time} & {\scriptsize Fail} &
	{\scriptsize Time} & {\scriptsize Fail} &
	{\scriptsize Time} & {\scriptsize Fail} &
	{\scriptsize Time} & {\scriptsize Fail} \\
	\hline
    Blocks & 7.47 & 0.00 & \bf{0.62} & 0.00 & 49.99 & 0.00 & \bf{0.52} & 0.00 & 7.25 & 0.74 & 2.33 & 0.00 & 52.95 & 0.23 \\
	Logistics & \bf{8.55} & 0.00 & \bf{6.44} & 0.00 & 42.05 & 0.00 & \bf{15.40} & 0.00 & -- & 1.00 & -- & 1.00 & 49.31 & 0.81 \\
	Miconic & 87.71 & 0.06 & \bf{21.64} & 0.04 & -- & 1.00 & 93.86 & 0.98 & -- & 1.00 & -- & 1.00 & -- & 1.00 \\
	Ferry & \bf{12.64} & 0.00 & \bf{7.52} & 0.00 & 43.79 & 0.10 & 39.66 & 0.00 & 34.78 & 0.91 & 33.77 & 0.00 & -- & 1.00 \\
	Gripper & 24.48 & 0.00 & \bf{0.47} & 0.00 & 56.58 & 0.29 & 37.63 & 0.00 & 28.94 & 0.60 & 5.71 & 0.20 & 86.29 & 0.95 \\
	Hanoi & \bf{3.19} & 0.00 & \bf{3.39} & 0.00 & \bf{3.47} & 0.00 & 4.63 & 0.00 & -- & 1.00 & 6.15 & 0.00 & 7.55 & 0.00 \\
	Exploding & 11.52 & 0.30 & \bf{0.81} & 0.30 & 44.96 & 0.32 & 1.08 & 0.29 & 10.18 & 0.89 & 4.69 & 0.19 & 48.53 & 0.38 \\
	Tireworld & 24.58 & 0.01 & \bf{4.38} & 0.08 & 44.09 & 0.29 & 47.13 & 0.00 & 30.36 & 0.10 & -- & 1.00 & 63.03 & 0.66 \\
	PyBullet & -- & 1.00 & \bf{2.05} & 0.00 & -- & 1.00 & 8.58 & 0.01 & -- & -- & -- & -- & -- & -- \\\hline
	\end{tabular}
	 \caption{On test problems, failure rates within a 120-second timeout and planning times in seconds over successful runs. All numbers report a mean across 10 random seeds, which randomizes both GNN training (if applicable) and testing. All times are in seconds; bolded times are within two standard deviations of best. See \tabref{tab:stds} in \appref{app:otherexps} for all standard deviations. AG = action grounding.  Policy and AG baselines are not run for  PyBullet because these methods cannot handle continuous actions. Across all domains, \ploi{} is consistently best and usually at least two standard deviations better than all other methods.}
	\label{tab:mainresults}
\end{table*}

\section{Experiments}
In this section, we present empirical results for \ploi{} and several baselines. We find that \ploi{} improves the speed of planning significantly over all these baselines.
See \appref{app:expdetails} for experimental details beyond those given here, and see \appref{app:otherexps} for additional experiments and results.

\subsection{Experimental Setup}
\textbf{Baselines.} We consider several baselines in our experiments, ranging from pure planning to state-of-the-art methods for learning to plan.
All GNN baselines are trained with supervised learning using the set of plans found by an optimal planner on small training problems.
\begin{tightlist}
\item \textbf{Pure planning.} Use the planner \plan{} on the complete test problems, with all the objects.
\item \textbf{Random object scoring.} Use the incremental procedure described in \secref{sec:approach}, but instead of using a trained GNN to score the importance of each object, give each object a uniformly random importance score between 0 and 1. This baseline can be understood as an ablation that removes the GNN from our system.
\item \textbf{Neighbors.} This is a simple heuristic approach that incrementally tries planning with all objects that are connected by at most $L$ steps in the graph of relations to any object named in the goal, for $L = 0, 1, 2, \ldots$. If a plan has not been found even after all objects connected to a goal object have been considered, we fall back to pure planning for completeness.
\item \textbf{Reactive policy.} Inspired by other works that learn reactive, goal-conditioned policies for planning problems~\cite{groshev2017learning,rivlin2020generalized}, we modify our GNN architecture to predict a ground action per timestep. The input remains the same, but the output has two heads: one predicts a probability over actions $\A$, and the other predicts, for every parameter of that action, a probability over objects.
At test time, we compute all valid actions in each state and execute the one with the highest probability under the policy. This baseline does not use \plan{} at test time.
\item \textbf{ILP action grounding.} This baseline is the method presented by \citet{gnad2019learning}, described in \secref{sec:related}, with the best settings they report. We use the implementation provided by the authors for both training and test. We use the SVR model with round robin queue ordering, and incremental grounding with increment 500.
\item \textbf{GNN action grounding.} We also investigate using a GNN in place of the inductive logic programming (ILP) model used by the previous baseline~\cite{gnad2019learning}. To implement this, we modify our GNN architecture to take as input a ground action in addition to the state and goal, and output the probability that this ground action should be considered when planning.
\end{tightlist}

As mentioned in \secref{sec:intro}, we also attempted to compare to a state-of-the-art lifted planner~\cite{correa2020lifted}, using the implementation provided by the authors. However, we found that this planner was unable to solve any of our test problems in any domain, although it was usually able to solve the (much smaller) training problems.

\textbf{Domains.}
We evaluate on 9 domains: 6 classical planning, 2 probabilistic planning, and 1 simulated robotic task and motion planning.
We chose several of the most standard classical and probabilistic domains from the International Planning Competition (IPC) \cite{ipc2008,ipc2014}, but we procedurally generated problems involving many more objects than is typical. In all domains, we train on 40 problem instances and test on 10 much larger ones. For interacting with IPC domains, we use the PDDLGym library~\cite{pddlgym}, version 0.0.2.
We describe each domain in \appref{app:domains}.
Here we report the total numbers of objects and the numbers of objects explicitly named in the goal for test problems in each domain:
\begin{tightlist}
\item \textbf{Tower of Hanoi}. 13-18 objects total (10-15 in goal).
\item \textbf{Blocks}. 100-150 objects total (20-25 in goal).
\item \textbf{Gripper}. 200-400 objects total (20-40 in goal).
\item \textbf{Miconic}. 2200-3200 objects total (100 in goal).
\item \textbf{Ferry}. 250-350 objects total (6 in goal).
\item \textbf{Logistics}. 130-160 objects total (40 in goal).
\item \textbf{Exploding Blocks}. Same as Blocks.
\item \textbf{Triangle Tireworld}. 2601-2809 objects total (1 in goal).
\item \textbf{PyBullet robotic simulation}~\cite{pybullet}. 1003 objects total (2 in goal). See \figref{fig:env} for details.
\end{tightlist}

\begin{figure}[t]
  \centering
    \noindent
    \includegraphics[width=\columnwidth]{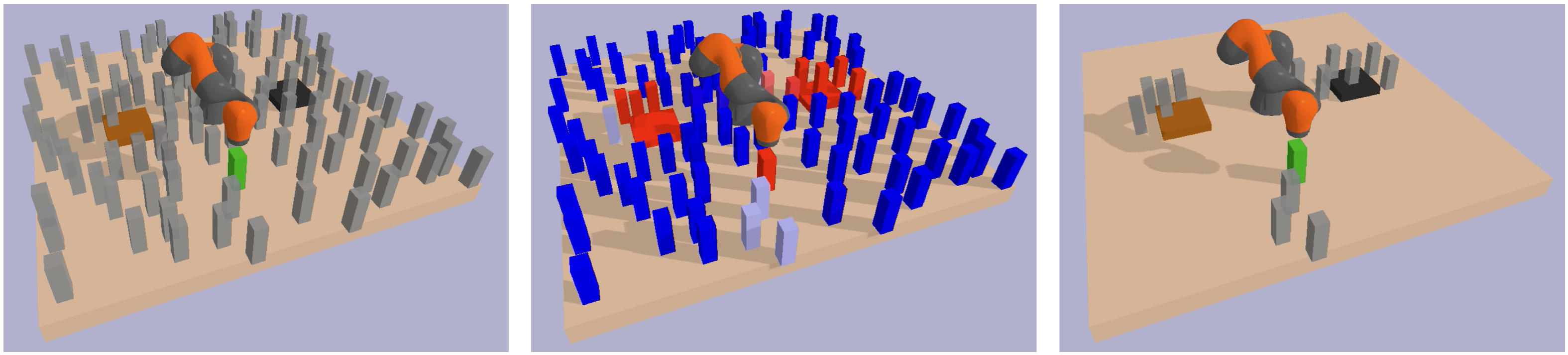}
    \caption{Example problem from the PyBullet domain. \emph{Left:} The robot arm must move the target can (green) to the stove (black) and then to the sink (brown) while avoiding other cans (gray). \emph{Middle:} GNN importance scores for this problem, scaled from blue (low importance) to red (high importance). We can see that cans surrounding the sink, stove, and target have been assigned higher importance score, meaning the GNN has reasoned about geometry. \emph{Right:} The reduced problem in which the robot plans. Only objects with importance score above some threshold remain in the scene.}
  \label{fig:env}
\end{figure}

\subsection{Results and Discussion}

All experiments are conducted over 10 random seeds. \tabref{tab:mainresults} shows failure rates within a 120-second timeout and average planning time on successful runs. Initial experimentation found no significant difference in our results between 300-second and 120-second timeouts. Across all domains, \ploi{} consistently plans much faster than all the other methods. In some domains, such as Gripper, \ploi{} is faster than \textbf{pure planning} by two orders of magnitude. In the case of Hanoi, where all objects are necessary, we see that \ploi{} is comparable to pure planning, which confirms the desirable property that \ploi{} reduces to pure planning with little overhead in problems where all objects are required.

Comparing \ploi{} with the \textbf{random object scoring} baseline, we see that \ploi{} performs much better in all domains other than Hanoi. 
This comparison suggests that the GNN is crucial for the efficient planning that \ploi{} attains. 
To further analyze the impact of the GNN, we plot the number of iterations ($N$ in \algref{alg:ploi}) that are needed until the incremental planning loop finds a solution, for both \ploi{} and random object scoring (\figref{fig:bar}). 
The dramatic difference between the two methods confirms that the GNN has learned a very meaningful bias, allowing a sufficient object set to be consistently found in less than 5 iterations, and often just 1.

\begin{figure}[t]
  \centering
    \noindent
    \includegraphics[width=\columnwidth]{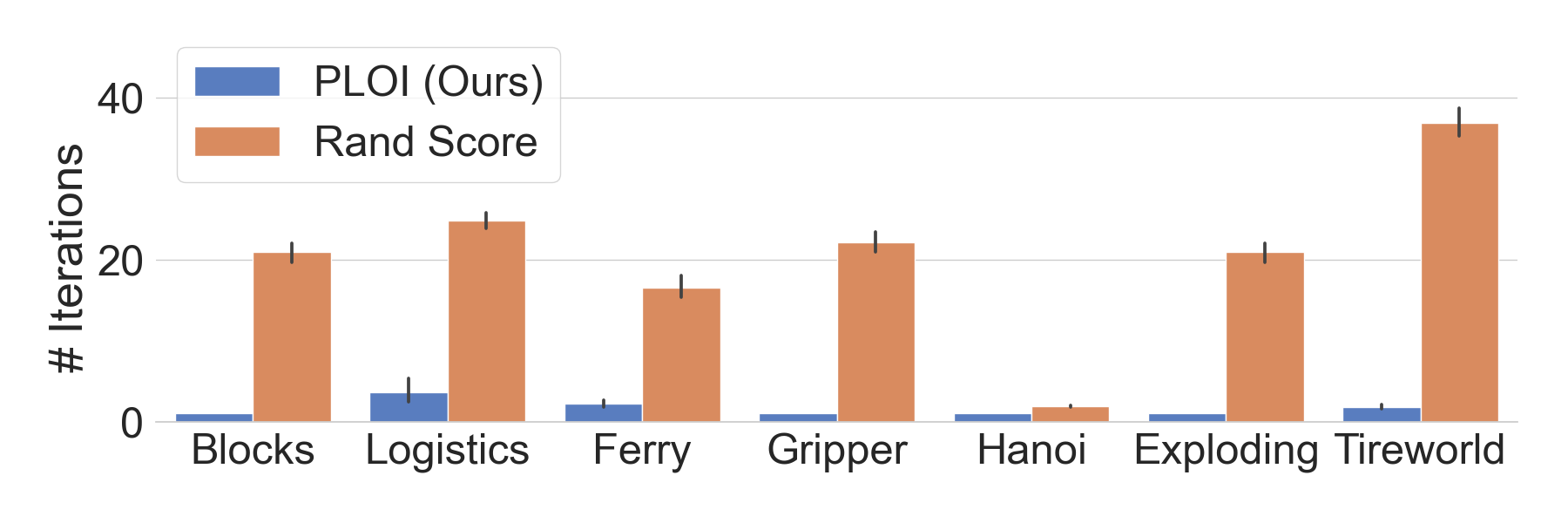}
    \caption{Number of iterations ($N$ in \algref{alg:ploi}) needed until incremental planning finds a solution, for both \ploi{} and random object scoring. Results are averaged over 10 seeds, with standard deviations shown as vertical lines. Miconic and PyBullet are not included because random object scoring never succeeded in this domain. We can see that the GNN has learned a meaningful bias, allowing a sufficient object set to be consistently found in fewer than 5 iterations.}
  \label{fig:bar}
\end{figure}

The key difference between \ploi{} and the \textbf{action grounding (AG)} baselines is that \ploi{} predicts which \emph{objects} would be sufficient for a planning problem, while the AG baselines predict which \emph{ground actions} would be sufficient for a planning problem. 
Empirically, \ploi{} performs better than all the AG baselines, due to the fact that \ploi{} has comparatively little overhead, while the AG baselines spend significant time during inference on trying to score all the possible ground actions, of which there are significantly more than the number of objects.
Another benefit of \ploi{} is that it uses \plan{} as a black box, whereas the AG baselines must modify the internals of \plan{}, e.g. by changing the set of ground actions instantiated during translation or followed during search.

The \textbf{neighbors} baseline performs well in some domains, but not in others; it performs particularly poorly in domains where the agent must consider an object that does not share a relation with some other important one, e.g. a ferry in the Ferry domain. Looking now at the \textbf{policy} baseline, we see that it is generally quite slow. This is because even though the policy baseline does not use \plan{}, it takes time to compute all valid actions and query the policy GNN to find the most probable one on every timestep; by contrast, \ploi{} only performs inference once, on the first timestep.
Moreover, the policy  has a high failure rate relative to the planning baselines, since there is no recourse when it does not succeed.

Finally, the results in the continuous PyBullet domain suggest that \ploi{} is able to yield meaningful improvements over an off-the-shelf task and motion planning system. 
Learning in the hybrid state and action spaces of task and motion planning domains is extremely challenging in general; reactive policy learning is typically unable to make meaningful headway in these domains.
Moreover, it is not possible to apply the action grounding approach due to the infinite number of ground actions (e.g., poses for grasping a can). \ploi{} works well here because it uses a planner in conjunction with making predictions at the level of the (discrete) object set, not the (continuous) ground action space.

\section{Conclusion}

We have introduced \ploi{}, a simple, powerful, and general mechanism for planning in large problem instances containing many objects. Empirically, we showed that \ploi{} performs well across classical planning, probabilistic planning, and robotic task and motion planning. As \ploi{} makes use of a neural learner to inform black-box symbolic planners, we view this work as a step toward the greater goal of integrated neuro-symbolic artificial intelligence~\cite{mao2019neuro,parisotto2016neuro,alshahrani2017neuro}.

An immediate direction for future work would be to investigate the empirical impact of using a GNN as the importance scorer, versus techniques in statistical relational learning~\cite{koller2007introduction,qu2019gmnn}. Another direction would be to study how to apply \ploi{} to open domains, where the agent does not know in advance the set of objects that are in a problem instance. 
Addressing this kind of future direction can help learning-to-plan techniques like \ploi{} fully realize their overarching aim of solving large-scale, real-world planning problems.

\newpage
\section*{Acknowledgements}
We would like to thank Kelsey Allen for valuable comments on an initial draft. We gratefully
acknowledge support from NSF grant 1723381; from AFOSR grant FA9550-17-1-0165; from ONR
grant N00014-18-1-2847; from the Honda Research Institute; from MIT-IBM Watson Lab; and from
SUTD Temasek Laboratories. Rohan and Tom are supported by NSF Graduate Research Fellowships.
Any opinions, findings, and conclusions or recommendations expressed in this material are those of
the authors and do not necessarily reflect the views of our sponsors.
\bibliography{biblio}

\clearpage
\appendix

\begin{table*}[t]
	\centering
	\footnotesize
	\begin{tabular}{| l | p{0.2cm} | p{0.8cm} | p{0.8cm} | p{0.8cm} | p{0.8cm} | p{0.8cm} | p{0.8cm} | p{0.8cm} | p{ 0.8cm} | p{0.8cm} | p{ 0.8cm} | p{0.8cm} | p{0.8cm} | }
	\hline
	\multicolumn{1}{|c|}{} &
	\multicolumn{1}{c|}{} &
	\multicolumn{3}{c|}{\ploi{} (Ours)} &
	\multicolumn{3}{c|}{Policy} &
	\multicolumn{3}{c|}{ILP AG} &
	\multicolumn{3}{c|}{GNN AG} \\
	\hline
	Domains & {\scriptsize \#T } &
	{\scriptsize Data} & {\scriptsize GNN} & {\scriptsize Total} &
	{\scriptsize Data} & {\scriptsize GNN} & {\scriptsize Total} &
	{\scriptsize Data} & {\scriptsize GNN} & {\scriptsize Total} &
	{\scriptsize Data} & {\scriptsize GNN} & {\scriptsize Total} \\
	\hline
    Blocks & 40 & 14.8 & 57 & 647 & 0.555 & 83 & 88.55 & -- & -- & 317.26 & 4.8 & 2990 & 3182  \\
	Logistics & 40 & 40.7 & 115 & 1742 & 12.5 & 230 & 355 & -- & -- & 954.75 & 14.08 & 77300 & 77863 \\
	Miconic & 40 & 39.3 & 568 & 2140 & 0.551 & 2820 & 2825.5 & -- & -- & 4906.5 & 0.52 & 120100 & 120120 \\
	Ferry & 40 & 10.2 & 230 & 637 & 0.445 & 452 & 456.45 & -- & -- & 894.60 & 0.37 & 15800 & 15814 \\
	Gripper & 40 & 33.9 & 130 & 1486 & 0.858 & 300 & 308.58 & -- & -- & 625.34 & 1.14 & 15270 & 15315   \\
	Hanoi & 6 & 1.1 & 53 & 59.63 & 0.34 & 231 & 233.05 & -- & -- &  22.59  & 0.35 & 4750 & 4752.1  \\\hline
	\end{tabular}
	 \caption{Training times for learning methods. For each learning method and for each domain, we report the total training time (``Total'' columns, seconds). For the GNN-based methods, we further report the breakdown between the GNN training time (``GNN'' columns, seconds) and the time required to create the training data \emph{per problem} (``Data'' columns, seconds), with the total number of training problems per domain reported on the left (``\#T'').}
	\label{tab:traintimes}
\end{table*}

\begin{table*}
	\centering
	\footnotesize
	\begin{tabular}{| l | p{0.7cm} | p{0.7cm} | p{0.7cm} | p{0.7cm} | p{0.7cm} | p{0.7cm} | p{0.7cm} | p{0.7cm} | p{0.7cm} | p{0.7cm} | p{0.7cm} | p{0.7cm} | p{0.7cm} | p{0.7cm} | }
	\hline
	\multicolumn{1}{|c|}{} &\multicolumn{2}{c|}{Pure Plan} &
	\multicolumn{2}{c|}{\ploi{} (Ours)} &
	\multicolumn{2}{c|}{Rand Score} &
	\multicolumn{2}{c|}{Neighbors} &
	\multicolumn{2}{c|}{Policy} &
	\multicolumn{2}{c|}{ILP AG} &
	\multicolumn{2}{c|}{GNN AG} \\
	\hline
	Domains &
	{\scriptsize Time} & {\scriptsize Fail} &
	{\scriptsize Time} & {\scriptsize Fail} &
	{\scriptsize Time} & {\scriptsize Fail} &
	{\scriptsize Time} & {\scriptsize Fail} &
	{\scriptsize Time} & {\scriptsize Fail} &
	{\scriptsize Time} & {\scriptsize Fail} &
	{\scriptsize Time} & {\scriptsize Fail} \\
	\hline
	Blocks & 0.07 & 0.00 & 0.07 & 0.00 & 15.80 & 0.00 & 0.06 & 0.00 & 0.77 & 0.35 & 0.04 & 0.00 & 27.68 & 0.32 \\
	Logistics & 0.05 & 0.00 & 5.95 & 0.00 & 2.68 & 0.00 & 0.53 & 0.00 & -- & 0.00 & -- & 0.00 & 14.45 & 0.17 \\
	Miconic & 2.94 & 0.05 & 4.32 & 0.14 & -- & 0.00 & -- & 0.06 & -- & 0.00 & -- & 0.00 & -- & 0.00 \\
	Ferry & 0.03 & 0.00 & 4.75 & 0.00 & 11.05 & 0.06 & 5.34 & 0.00 & -- & 0.28 & 1.82 & 0.00 & -- & 0.00 \\
	Gripper & 0.17 & 0.00 & 0.03 & 0.00 & 15.67 & 0.21 & 5.78 & 0.00 & 2.73 & 0.37 & 0.12 & 0.00 & 22.69 & 0.10 \\
	Hanoi & 0.17 & 0.00 & 0.39 & 0.00 & 0.32 & 0.00 & 0.41 & 0.00 & -- & 0.00 & 0.19 & 0.00 & 0.19 & 0.00 \\
	Exploding & 2.97 & 0.14 & 0.12 & 0.15 & 18.53 & 0.17 & 0.12 & 0.11 & 2.37 & 0.16 & 1.05 & 0.07 & 29.10 & 0.13 \\
	Tireworld & 9.11 & 0.04 & 1.23 & 0.22 & 12.06 & 0.08 & 4.98 & 0.00 & 13.21 & 0.21 & -- & 0.00 & 15.36 & 0.21 \\
	PyBullet & 0.00 & 0.00 & 0.11 & 0.00 & 0.00 & 0.00 & 0.21 & 0.03 & -- & -- & -- & -- & -- & -- \\
	\hline
	\end{tabular}
	\caption{Standard deviations for main results. See Table \ref{tab:mainresults} in the main text for means and experimental details.}
	\label{tab:stds}
\end{table*}

\section{Experimental Details}
\label{app:expdetails}
\emph{Planning details.} We use Fast Downward~\cite{fd} in the \texttt{LAMA-first} mode as the base classical planner for test time in all experiments.
To gather training data with an optimal planner, we use Fast Downward in \texttt{seq-opt-lmcut} mode.
For planning in probabilistic domains, we use single-outcome determinization and replanning \cite{ffreplan}.
For TAMP in the PyBullet experiment, we use PDDLStream~\cite{garrett2020pddlstream} in \texttt{focused} mode.

\emph{Hardware details.} All experiments were performed on Ubuntu 18.04 using four cores of an Intel Xeon Gold 6248 processor, with 10GB RAM per core.

\emph{\ploi{} details.} We use $\gamma=0.9$ for all experiments.

\emph{GNN details.}
GNNs are implemented in PyTorch, version 1.5.0.
All GNNs node and edge modules are fully connected neural networks with one hidden layer of dimension 16, ReLU activations, and layer normalization \cite{layernorm}.
Message passing is performed for $K=3$ iterations.
Training uses the Adam optimizer with learning rate $0.001$ for 1000 epochs. The batch size is 16.
Preliminary experiments with $\ell_2$ regularization, dropout, and hyperparameter search yielded no consistent improvements for any of the methods.

\section{Domain Descriptions}
\label{app:domains}
We evaluate on 9 domains: 6 classical planning, 2 probabilistic planning, and 1 simulated robotic task and motion planning.
The classical and probabilistic domains are from the International Planning Competition (IPC) \cite{ipc2008,ipc2014}.
\begin{itemize}
\item \textbf{Tower of Hanoi}. The classic Tower of Hanoi domain, in which disks must be moved among three pegs. All objects are always necessary to consider in this domain; we have included this domain to show that \ploi{} does not have much overhead on top of pure planning in this situation. We train on problems containing 4-9 disks and test on problems containing 10-15 disks.
The plan lengths for training (test) problems range from 1-63 (511-8191).
\item \textbf{Blocks}. Problems involve blocks in small piles on a table, and the goal is to configure a particular small subset of the blocks into a tower.
We train on problems containing 15-32 blocks and test on problems containing 100-150 blocks.
Test goals involve 20-25 blocks.
The plan lengths for training (test) problems range from 4-10 (26-64).
\item \textbf{Gripper}. Problems involve one robot that can pick and place balls and move to different rooms.
A goal is an assignment of a subset of the balls to rooms.
We train on problems containing 36-52 objects and test on problems containing 100-200 objects.
Test goals involve placing 10-20 balls in random rooms.
The plan lengths for training (test) problems range from 7-19 (27-107).
\item \textbf{Miconic}. Passengers in buildings with elevators are trying to reach particular floors.
We train on problems involving 33-63 objects.
We test on problems with 20-30 floors, 2 passengers per building, and 100 buildings, for a total of over 2000 objects.
Goals involve moving one passenger per building to their desired floor.
The plan lengths for training (test) problems range from 11-12 (894-917).
\item \textbf{Ferry}. A ferry transports cars to various locations. 
We train on problems with 13-21 objects and test on problems with 250-350 objects.
Goals involve moving 3 cars to random locations.
The plan lengths for training (test) problems range from 7-12 (14-17).
\item \textbf{Logistics}. Trucks and airplanes are used to transport crates to cities.
We train on problems with 13-40 objects.
Test problems have around 50 airplanes, 20 cities, 20 trucks, 20-50 locations, and 20 crates. Goals involve moving around 20 crates to random cities.
The plan lengths for training (test) problems range from 5-32 (66-203).
\item \textbf{Exploding Blocks}. A probabilistic IPC domain, where whenever the agent interacts with a block, there is a chance that the block or the table are irreversibly destroyed; no policy can succeed all the time in this domain. Problem sizes are the same as in Blocks.
\item \textbf{Triangle Tireworld}. A probabilistic IPC domain, containing an agent that must navigate through cities, and has a chance of getting a flat tire on each timestep. The agent can only change its tire at certain cities that have spare tires. It is always possible to reach the goal city by simply avoiding cities that do not have spare tires. 
We test on worlds with side length around 50.
\item \textbf{PyBullet robotic simulation}~\cite{pybullet}. In this domain with continuous object properties, a fixed robot arm mounted on the center of a table must interact with a particular can on the table while avoiding all other irrelevant cans. See \figref{fig:env} for details and a visualization. To encode this domain in our GNN, we treat the continuous object poses as node features. Test problems have around 1000 irrelevant cans on the table. The goal always involves manipulating a single can.
\end{itemize}

\section{Additional Experiments}
\label{app:otherexps}

Here we report additional experiments and results.

\subsubsection{Effect of Message Passing Iterations ($K$)}

\begin{figure}[h]
  \centering
    \noindent
    \includegraphics[width=\columnwidth]{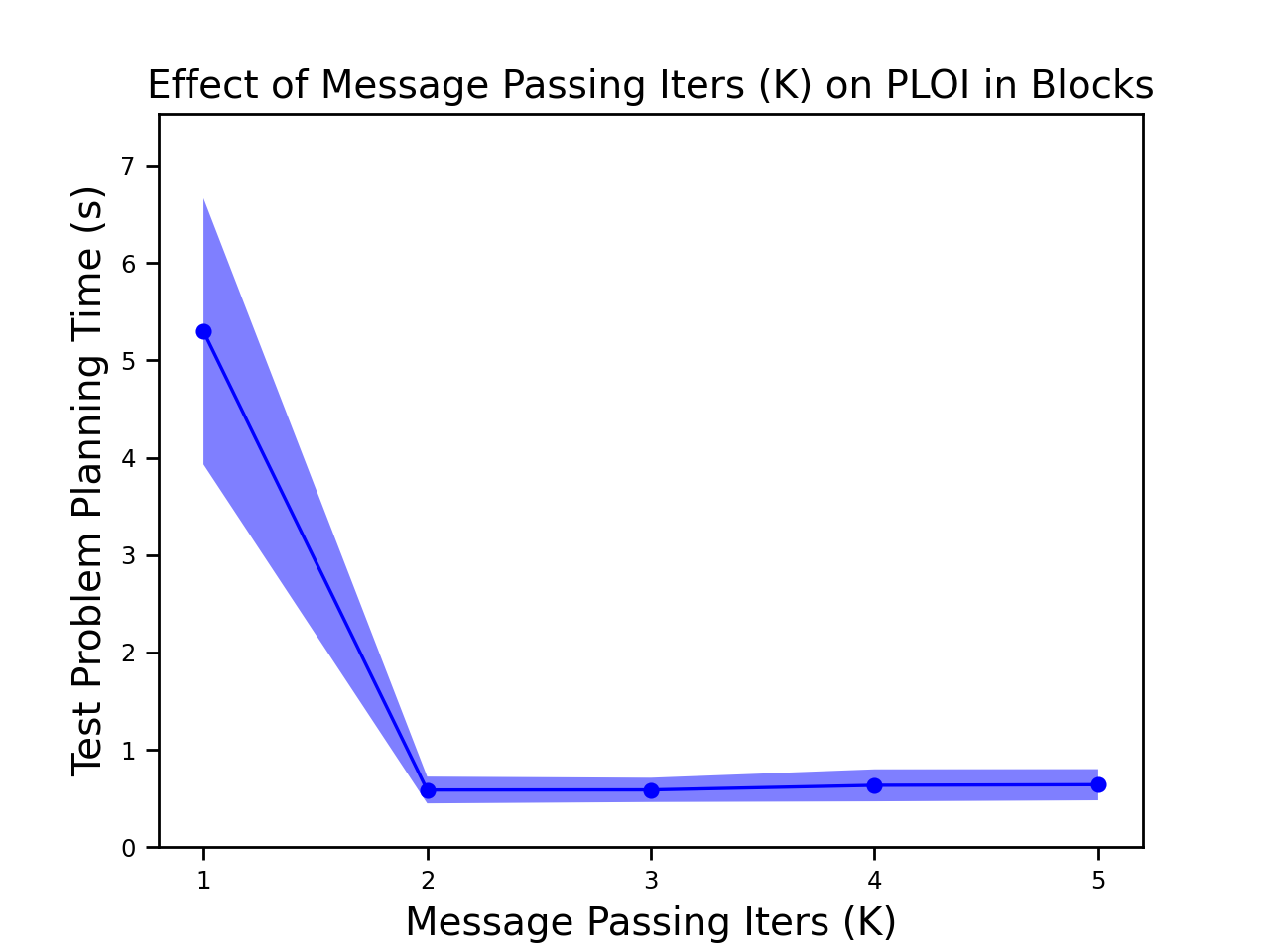}
    \caption{Effect of message passing iterations ($K$) on the performance of \ploi{} in Blocks. Results are averaged over 10 seeds, with standard deviations shown as shaded areas.}
  \label{fig:messages}
\end{figure}

We used $K=3$ message passing iterations for all graph neural networks.
To better understand the impact of this hyperparameter on our main results, we reran \ploi{} on Blocks, varying $K$ from 1 to 5.
As seen in \figref{fig:messages}, results are robust for $2 \le K \le 5$, but performance drops off heavily for $K=1$, suggesting that some propagation through the GNN matters.
In other domains, we would similarly expect $K=1$ to be insufficient, but we may not always expect $K=2$ to suffice. Generally, setting $K$ appropriately involves a trade-off: too low values may prevent the model from fitting the data, while too large values may slow computation and risk overfitting.
A hyperparameter search increasing from $K=1$ should do well to identify an appropriate value for any domain.

\subsubsection{Effect of Number of Training Problems}

\begin{figure}[h]
  \centering
    \noindent
    \includegraphics[width=\columnwidth]{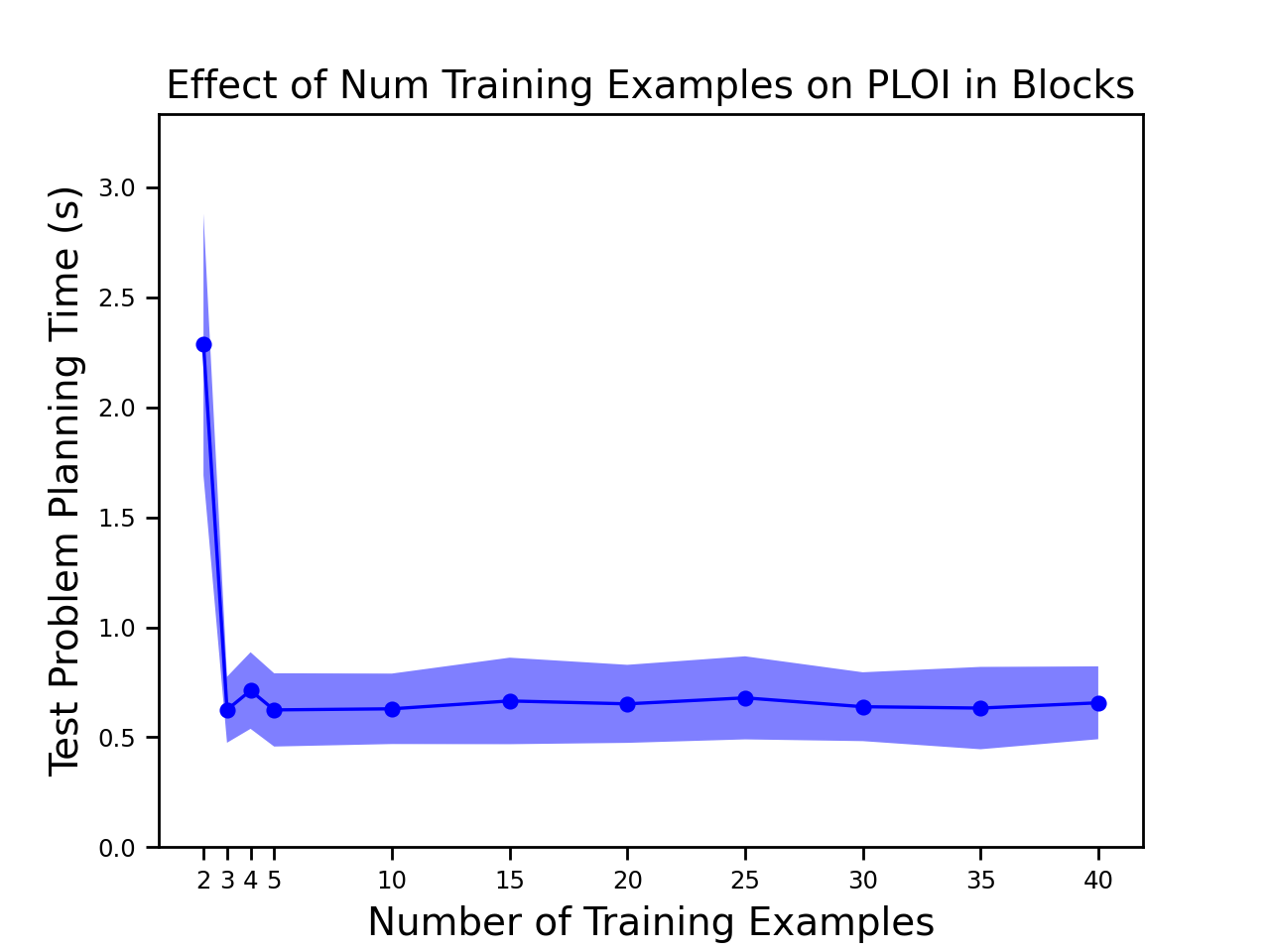}
    \caption{Effect of number of training examples on the performance of \ploi{} in Blocks. Results are averaged over 10 seeds, with standard deviations shown as shaded areas.}
  \label{fig:numexamples}
\end{figure}

We used $<50$ training problems in all domains, with 40 used in Blocks.
To better understand the impact of the number of training problems on our main results, we reran \ploi{} on Blocks, varying the number of training problems between 2 and 40.
As seen in \figref{fig:numexamples}, performance peaks very quickly, starting at 3 and remaining robust for $>3$. We would not necessarily expect so few examples to suffice for the other domains.

\subsubsection{Training Times}

Our main results compare the time required at test time for \ploi{} and baselines to plan.
In Table \ref{tab:traintimes}, we report the time required by \ploi{} and the other learning methods at training time, with a breakdown between training dataset generation and GNN training where applicable.
Our findings are: (1) data generation (predominantly data labelling) takes from 15 seconds (Blocks) to 90 seconds (Tireworld) per problem; (2) training the GNN for \ploi{} is much faster than for Policy or GNN AG, averaging 3 min for \ploi{}, 5 min for Policy, and 500 min for GNN AG; (3) training speed for \ploi{} is on par with that for ILP AG \cite{gnad2019learning}. The difference in (2) is because \ploi{} needs much less data than Policy or GNN AG, since \ploi{} does not operate on actions and is only run on the initial state.

\end{document}